\documentclass[journal]{IEEEtran}

\usepackage{amsmath, amsthm, amssymb, mathtools, bm}
\usepackage{lmodern}
\usepackage{xcolor}
\usepackage{graphicx}
\usepackage{algorithm}
\usepackage{algpseudocode}
\usepackage{enumerate}
\usepackage{array}
\usepackage{booktabs}
\usepackage{hyperref}
\usepackage{zref-clever}
\zcsetup{cap}
\zcRefTypeSetup{figure}{
    Name-sg = Fig.,
    name-sg = Fig.,
    Name-pl = Figs.,
    name-pl = Figs.
}
\zcRefTypeSetup{definition}{
    Name-sg = Def.,
    name-sg = Def.,
    Name-pl = Defs.,
    name-pl = Defs.
}
\zcRefTypeSetup{equation}{
    Name-sg = Eq.,
    name-sg = Eq.,
    Name-pl = Eqs.,
    name-pl = Eqs.
}
\zcRefTypeSetup{appendix}{
    Name-sg = Appendix,
    name-sg = appendix,
    Name-pl = Appendices,
    name-pl = appendices
}
\usepackage{xspace}
\usepackage{microtype} 

\usepackage{tikz}
\usetikzlibrary{shapes.geometric, arrows, positioning, shapes.symbols}

\newtheorem{definition}{Definition}
\newtheorem{lemma}{Lemma}
\newtheorem{theorem}{Theorem}
\newtheorem{corollary}{Corollary}

\newtheorem{remark}{Remark}

\newcommand{\br}[1]{\left(#1\right)}

\newcommand{\norm}[1]{\left\lVert{#1}\right\rVert}

\DeclareMathOperator*{\Exp}{\mathbf{E}}

\DeclareMathOperator{\clip}{clip}

\newcommand{\R}{\mathbb{R}} 
\newcommand{\Hil}{\mathcal{H}} 
\newcommand{\algname}{{DECO-EF}\xspace}
\newcommand{\yesfeat}{\textcolor{green!50!black}{\ensuremath{\checkmark}}}
\newcommand{\nofeat}{\textcolor{red!70!black}{\ensuremath{\times}}}
\newcommand{\alphastoch}{\alpha^{\mathrm{stoch}}}
\newcommand{\etastoch}{\bar{\eta}}
\newcommand{\Pperp}{P_{\perp}}

\title{Decentralized Parameter-Free Online Learning with Compressed Gossip}

\author{Tomas Ortega, \IEEEmembership{Graduate Student Member, IEEE}, and Hamid Jafarkhani, \IEEEmembership{Fellow, IEEE}
    \thanks{
        Authors are with the Center for Pervasive Communications \& Computing and  EECS Department, University of California, Irvine, Irvine, CA 92697 USA (e-mail: \{tomaso, hamidj\}@uci.edu).
        This work was supported in part by the NSF Award ECCS-2207457.
}
}

\begin{document}
\bstctlcite{IEEEtran:max5names}
\maketitle


\begin{abstract}
    We study decentralized online convex optimization when agents communicate over a graph and messages may be compressed.
    Classical decentralized online methods typically require learning-rate choices that depend on the horizon, comparator scale, or other problem parameters, while compressed communication introduces additional disagreement that must be controlled.
    We propose \algname{} (DEcentralized COin-betting with Error Feedback), a decentralized parameter-free online learning algorithm that combines coin-betting predictions with compressed difference-based gossip.
    Each agent maintains a clean accumulated state and a compressed tracker, and communicates only compressed state differences during gossip steps.
    The method is parameter-free in the online-learning sense: it does not tune to the horizon, the comparator norm, or the learning rate.
    We prove expected comparator-adaptive network-regret bounds for \algname{} under compressed communication.
    To the best of our knowledge, this gives the first expected sublinear network-regret guarantees for parameter-free decentralized online learning under compressed communication.
\end{abstract}

\section{Introduction}
\IEEEPARstart{N}{etworked} signal processing and learning systems increasingly make decisions from data streams instead of stored data sets.
Online Convex Optimization (OCO) provides a standard adversarial model for such sequential decisions: at each round, a learner chooses a point before observing the current convex loss and performance is measured by regret with respect to the best fixed comparator in hindsight~\cite{hazan_introduction_2023,shalev-shwartz_online_2011}.

Two practical issues arise in this context.
First, classical OCO methods such as online gradient descent (OGD) and follow-the-regularized-leader (FTRL) require learning-rate choices that depend on the horizon, the comparator norm, or loss-scale parameters~\cite{hazan_introduction_2023}.
These quantities are usually unknown before the stream is observed, and a mis-specified learning rate can cause catastrophic losses~\cite{orabona2020icml}.
Parameter-free OCO addresses this issue by replacing tuned learning rates with comparator-adaptive regret guarantees~\cite{orabona_modern_2023}.
A parameter-free guarantee is especially useful in streaming and edge deployments: the same algorithm can be run without a horizon estimate, without a comparator-scale guess, and without an expensive learning-rate sweep for every graph, device class, or data stream.
A prominent mechanism is coin betting, which proves regret by lower-bounding a sequential wealth process and then converting this wealth certificate through a Fenchel conjugate~\cite{mcmahan_no-regret_2012,mcmahan_unconstrained_2014,orabona_coin_2016,cutkosky_online_2017,van_der_hoeven_comparator-adaptive_2020}.
For unbounded domains with Lipschitz losses, these methods obtain the optimal dependence on the comparator norm up to logarithmic factors~\cite{mcmahan_unconstrained_2014}.

Second, modern learning systems are often decentralized.
Each agent observes only a local loss stream, communicates only with neighbors in a graph, and nevertheless contributes decisions that should be good for the network as a whole~\cite{rabbat_distributed_2004,johansson_simple_2007,nedic_rate_2007,nedic_distributed_2018,cesa2020cooperative,cesa2021cooperative}.
Communication is also lossy in many signal-processing and edge-learning deployments.
The difficulty is that parameter-free learning and compressed decentralized communication interact in a nontrivial way.
Coin-betting predictions are nonlinear functions of accumulated gradients, so small disagreement in the accumulated state can be amplified by the prediction map.
At the same time, biased compressors create tracking error that ordinary gossip does not remove unless the communication protocol explicitly corrects it.

This paper studies whether comparator-adaptive OCO guarantees can be retained when all communication is decentralized and compressed.
We consider \(N\) agents on a connected graph.
At Round \(t\), Agent \(n\) predicts \(x_{n,t}\), observes a local convex loss \(\ell_{n,t}\), and exchanges compressed messages with its neighbors.
The objective is to minimize the \emph{network regret}, defined as follows.
For a comparator \(u\) in a Hilbert space \(\Hil\), the network regret after $T$ OCO rounds is
\begin{equation}
    R_T^{\mathrm{net}}(u)
    =
    \sum_{t=0}^{T-1}\frac{1}{N}\sum_{n=1}^N \bar\ell_t(x_{n,t})
    -
    \sum_{t=0}^{T-1}\bar\ell_t(u),
    \label{eq:network-regret}
\end{equation}
where
\begin{equation}
    \bar\ell_t(x)=\frac1N\sum_{n=1}^N\ell_{n,t}(x)
    \label{eq:network-loss}
\end{equation}
is defined as the \emph{network loss} at Round \(t\).
The network regret evaluates all local predictions under the averaged network loss and compares them with one fixed comparator \(u\).
This is the natural collaborative benchmark for a network that learns a shared decision from distributed local observations.\footnote{%
    We take the name from~\cite{achddou_distributed_2024}, related work uses \emph{collective} regret~\cite{hsieh_multi-agent_2022}.}

We propose \algname{} (DEcentralized COin-betting with Error Feedback), a decentralized parameter-free OCO algorithm with compressed gossip.
The term parameter-free is used in the standard OCO sense: the algorithm does not tune a learning rate to the horizon or comparator norm.
The graph-dependent gossip step size and the number of compressed gossip substeps remain communication parameters that depend only on the graph and compressor properties, not on the horizon, loss sequence, or comparator.

The main contributions are as follows.
\begin{itemize}
    \item We introduce \algname{}, a family of decentralized coin-betting algorithms in which agents communicate compressed state differences.
    \item We prove a network-regret decomposition that separates the centralized Hilbert-space coin-betting term from the prediction disagreement caused by decentralization and compression, and establish expected compressed-gossip tracking bounds for randomized mean-square compressors.
    \item We prove that \algname{} with Krichevsky-Trofimov (KT) and shifted-exponential coin-betting potentials achieves expected \(\widetilde O(\max\{1,\norm u\}\sqrt T)\) network-regret bounds under a linear compressed-gossip schedule.
    \item We validate our results on both synthetic and real data, illustrating the trade-offs between communication and accuracy for \algname{}.
\end{itemize}
To the best of our knowledge, we present the first algorithm family to achieve expected sublinear network-regret guarantees for parameter-free decentralized online learning with compressed synchronous gossip messages.


The remainder of the paper is organized as follows.
\zcref{sec:related-work} reviews related work.
\zcref{sec:model-setup} defines the communication assumptions.
\zcref{sec:algorithm} describes the high-level algorithm logic, prediction rule, and \algname{} update.
\zcref{sec:main-results} states the main regret guarantees.
\zcref{sec:analysis} provides the proof idea for the main regret guarantees, with technical details in the appendices.
\zcref{sec:operational-details} records implementation choices such as the gossip schedule, step size, and compressor.
\zcref{sec:experiments} describes the empirical evaluation and \zcref{sec:conclusion} concludes the paper.

\section{Related Work}\label{sec:related-work}
\paragraph{Online convex optimization and parameter-free learning}
OCO formalizes the worst-case sequential prediction through regret minimization~\cite{hazan_introduction_2023,shalev-shwartz_online_2011}.
Classical algorithms such as OGD, mirror descent, dual averaging, and FTRL achieve sublinear regret with appropriately tuned step sizes, but their guarantees depend on parameters such as the horizon, gradient scale, domain diameter, or comparator norm.
Parameter-free online learning removes this tuning by proving comparator-adaptive bounds.
Coin-betting methods are a central approach: they maintain a wealth process, select predictions as bets, and obtain regret by duality from the wealth lower bound~\cite{mcmahan_no-regret_2012,mcmahan_unconstrained_2014,orabona_coin_2016,cutkosky_online_2017,van_der_hoeven_comparator-adaptive_2020}.
Our analysis uses this machinery, but the regret is for a network of learners rather than a single agent.

\paragraph{Decentralized online optimization}
Distributed and decentralized optimization over graphs is a long-standing topic in signal processing, control, and machine learning~\cite{rabbat_distributed_2004,johansson_simple_2007,duchi_dual_2012,nedic_decentralized_2015,nedic_network_2018}.
In online variants, agents observe time-varying losses and communicate through graph-supported averaging or dual-averaging steps, and the resulting regret bounds depend on the spectral properties of the graph and problem-specific parameters~\cite{hosseini_online_2013,nedic_rate_2007}.
Recent parameter-free distributed online-learning results address related but different models, including delayed information, activated agents, and the availability of random agents~\cite{hoeven_distributed_2021,achddou_distributed_2024,hsieh_multi-agent_2022}.
The present work focuses on the synchronous and fully decentralized setting in which every agent acts each round, communication is graph-local, and the benchmark is network regret.

\paragraph{Compressed decentralized online learning}
Communication-efficient distributed learning uses quantization, sparsification, and other lossy compressors to reduce message sizes~\cite{qsgd,sparsifiedSGD,advances_open_problems}.
Because biased compressors such as Top-\(k\) can accumulate error, error-feedback (EF) methods track residuals or differences rather than compressing full updates~\cite{error_feedback,EF21,unbiased_horvath,Ortega_Jafarkhani_2023_asynch,Ortega_Jafarkhani_2024a}.
In decentralized optimization, several methods combine this difference-tracking idea with compressed proxies and gossip dynamics to obtain consensus and optimization convergence under lossy communication~\cite{koloskova_decentralized_2019,koloskova_decentralized_2020}.
Related online-learning work has studied collaborative regret under limited communication, including quantized online multi-kernel learning and gossiped online prediction~\cite{shen2021distributed,Ortega_Jafarkhani_2023_gossiped}, as well as compressed federated or distributed optimization variants~\cite{Ortega_Huang_Li_Jafarkhani_2024,ortega2025communicationcompressiondistributedlearning}.
The closest predecessor of the present work is the uncompressed DECO (DEcentralized COin-betting) algorithm~\cite{Ortega_Jafarkhani_2025}, which gives parameter-free decentralized OCO with comparator-adaptive network-regret
guarantees under exact communication.
Here we add compressed difference-tracking gossip and the corresponding tracking analysis, obtaining expected comparator-adaptive network-regret bounds under compressed communication.
Parameter-free compressed algorithms were previously obtained for the weaker joint-regret criterion~\cite{hoeven_distributed_2021,hsieh_multi-agent_2022}, while other communication-constrained online-learning methods do not provide comparator-adaptive network regret~\cite{Yang_Yang_Jiang_Zhang_2025,CB_2023}.
\algname{} combines these directions by importing the difference-tracked compressed gossip into adversarial decentralized OCO, preserving the comparator-adaptive network regret while allowing compressed messages.

\zcref{tab:algorithm-comparison} summarizes the related work.
The rows differ in the feedback model, synchrony, and objective.
Its purpose is to isolate the three properties combined here: decentralized communication, compressed gossip, and comparator-adaptive sublinear network regret.

\begin{table*}[t]
    \caption{Positioning of representative algorithm families.
        Comparator-adaptive (CA) denotes explicit dependence on \(\norm{u}\).}%
    \label{tab:algorithm-comparison}
    \centering
    \setlength{\tabcolsep}{3.5pt}
    \renewcommand{\arraystretch}{1.14}
    \resizebox{\textwidth}{!}{%
        \begin{tabular}{@{}lccccl@{}} 
            \toprule
            Method family & Decentralized & Compressed gossip & CA         & Sublinear Network Regret & Representative guarantee   \\
            \midrule
            OGD/FTRL~\cite{hazan_introduction_2023,shalev-shwartz_online_2011}
                          & \nofeat{}     & \nofeat{}         & \nofeat{}  & \yesfeat{}               & tuned network regret       \\
            Coin betting~\cite{mcmahan_no-regret_2012,orabona_coin_2016}
                          & \nofeat{}     & \nofeat{}         & \yesfeat{} & N/A                      & single-agent CA regret     \\
            Distributed CA OCO~\cite{hoeven_distributed_2021}
                          & \yesfeat{}    & \yesfeat{}        & \yesfeat{} & \nofeat{}                & joint regret               \\
            Decentralized OGD~\cite{johansson_simple_2007,nedic_rate_2007,nedic_decentralized_2015}
                          & \yesfeat{}    & \nofeat{}         & \nofeat{}  & \yesfeat{}               & tuned network regret       \\
            EF-style compression~\cite{koloskova_decentralized_2019,EF21}
                          & \yesfeat{}    & \yesfeat{}        & \nofeat{}  & N/A                      & optimization convergence   \\
            DECO~\cite{Ortega_Jafarkhani_2025}
                          & \yesfeat{}    & \nofeat{}         & \yesfeat{} & \yesfeat{}               & CA network regret          \\
            \algname{}
                          & \yesfeat{}    & \yesfeat{}        & \yesfeat{} & \yesfeat{}               & expected CA network regret \\
            \bottomrule
        \end{tabular}%
    } %
\end{table*}

The network-regret notion in \zcref{eq:network-regret} is a global-performance criterion.
It differs from per-agent regret, which certifies each node only against its own observed loss stream, and from consensus error, which does not by itself measure predictive performance.
As noted in~\cite{achddou_distributed_2024}, it is also different from the average of per-agent regrets, which is a weaker criterion that can be sublinear even if the network regret is not.
The uncompressed DECO algorithm~\cite{Ortega_Jafarkhani_2025} is recovered from \algname{} by taking the identity compressor and the exact-message communication.

\section{Setup}\label{sec:model-setup}
This section provides the communication graph and compressor definitions we need to set up the proofs.


\subsection{Communication Graph and Compression}\label{subsec:communication-compression}
Agents communicate on a fixed undirected communication graph \(\mathcal G=([N],E)\), which we assume to be connected.
A valid gossip matrix for \(\mathcal G\) is a nonnegative, symmetric, doubly stochastic matrix \(W\) whose off-diagonal support is exactly the edge set:
\begin{equation*}
    W_{ij}>0 \text{ if } \{i,j\}\in E,\quad
    W_{ij}=0 \text{ if } i\neq j \text{ and } \{i,j\}\notin E .
\end{equation*}
We also require positive self-weights \(W_{ii}>0\).
For \(Z\in\Hil^N\), we define the product network norm
\begin{equation}
    \norm{Z}_{2,\Hil}
    =
    \left(\sum_{n=1}^N\norm{Z_n}^2\right)^{1/2},
    \label{eq:product-network-norm}
\end{equation}
and we define
\begin{equation*}
    {(WZ)}_n=\sum_{j=1}^N W_{nj}Z_j .
\end{equation*}
For such a gossip matrix, the usual finite-dimensional spectral-gap theorem~\cite{boyd_randomized_2006} provides a zero-mean contraction factor \(\rho\in[0,1)\). 
Equivalently, the spectral gap \(\mu=1-\rho\) satisfies \(0<\mu\leq1\), with \(\rho=0\) for exact averaging.
The Hilbert lift of \(W\) is nonexpansive in this product \(\ell_2\) norm and contracts the zero-mean subspace:
\begin{equation}
    \begin{aligned}
        \norm{WZ}_{2,\Hil}
         & \leq \norm{Z}_{2,\Hil}
         &                            & \text{for all }Z\in\Hil^N,     \\
        \norm{WZ}_{2,\Hil}
         & \leq \rho\norm{Z}_{2,\Hil}
         &                            & \text{if }\sum_{n=1}^N Z_n=0 .
    \end{aligned}
\end{equation}

Messages sent in our graph are compressed.

\begin{definition}[Compressor]\label{def:compressor}
    A randomized compressor on \(\Hil\) consists of a sample space \(\Omega\), a map \(\mathcal C:\Omega\times\Hil\to\Hil\), and a parameter \(0\leq\alphastoch<1\) such that for every input \(v\in\Hil\),
    \begin{equation}
        \Exp_\omega\norm{\mathcal C(\omega,v)-v}^2
        \leq \alphastoch\norm v^2 .
        \label{eq:compressor-alpha-stoch}
    \end{equation}
\end{definition}

Fix a compressor in the sense of \zcref{def:compressor}.
\zcref{tab:compressors} lists representative mean-square residual coefficients \(\alphastoch\) for compressors with inputs in \(v\in\R^d\).
For the low-rank approximation entry, \(v\) is reshaped to an \(a\times b\) matrix with \(d=ab\), \(p=\min\{a,b\}\), and the compressor keeps the best rank-\(r\) approximation with \(r<p\).
For a \(b\)-bit uniform scalar quantization, one bit stores the sign and the remaining \(b-1\) bits store the magnitude: after normalizing by \(\norm{v}_{\infty}\), each coordinate magnitude is rounded to one of the \(2^{b-1}\) equally spaced values in \([0,1]\).
Top-\(k\) sparsification keeps the \(k\) largest-magnitude coordinates and replaces the rest with zeros. The random-\(k\) sparsification chooses the \(k\) coordinates to keep uniformly at random.

\begin{table}[t]
    \caption{Representative mean-square compressor coefficients.}%
    \label{tab:compressors}
    \centering
    \small
    \setlength{\tabcolsep}{2.3pt}
    \renewcommand{\arraystretch}{1.15}
    \begin{tabular}{@{}lc@{}}
        \toprule
        Compressor                                       & \(\alphastoch\)                \\
        \midrule
        Identity~\cite{Ortega_Jafarkhani_2025}           & \(0\)                          \\
        Top-\(k\)~\cite{sparsifiedSGD,qsgd}              & \(1-k/d\)                      \\
        Low-rank, rank \(r<p\)~\cite{vogels2019powersgd} & \(1-r/p\)                      \\
        \(b\)-bit quant.~\cite{qsgd}                     & \(\frac{d}{4{(2^{b-1}-1)}^2}\) \\
        Random-\(k\), unscaled~\cite{sparsifiedSGD}      & \(1-k/d\)                      \\
        \bottomrule
    \end{tabular}
\end{table}

\section{Algorithm}\label{sec:algorithm}
This section describes the algorithmic idea and the concrete \algname{} update analyzed in the paper.

\subsection{High-Level Idea}\label{subsec:high-level-idea}
\algname{} combines parameter-free coin-betting prediction with compressed error-feedback gossip.
A complete description of a round of update appears in \zcref{alg:deco-ef}, and its high-level system model is sketched in \zcref{fig:system-model}.
\begin{figure}[t]
    \centering
    \begin{tikzpicture}[
        font=\scriptsize,
        agent/.style={circle, draw=black, fill=gray!10, align=center,
                minimum size=1.05cm, inner sep=1pt},
        block/.style={rectangle, draw=black, rounded corners=2pt,
                fill=blue!7, align=center, inner sep=3pt},
        update/.style={rectangle, draw=black, rounded corners=2pt,
                fill=orange!10, align=center, inner sep=3pt},
        target/.style={rectangle, draw=black, rounded corners=2pt,
                fill=green!10, align=center, inner sep=3pt},
        msg/.style={->, thick, >=stealth},
        gossip/.style={<->, thick, dotted, >=stealth}
    ]
    \node[block] (losses) at (0,1.5)
    {local loss streams\\\(\ell_{n,t}\) and \(g_{n,t}\)};

    \node[agent] (a1) at (-2.5,0) {Agent\\\(1\)};
    \node[agent] (a2) at (0,0) {Agent\\\(n\)};
    \node[agent] (a3) at (2.5,0) {Agent\\\(N\)};

    \draw[gossip] (a1) -- (a2);
    \draw[gossip] (a2) -- (a3);
    \draw[gossip] (a1) to[bend right=20]
    node[below] {} (a3);

    \draw[msg] (losses.south west) -- (a1.north);
    \draw[msg] (losses.south) -- (a2.north);
    \draw[msg] (losses.south east) -- (a3.north);

    \node[update] (states) at (0,-1.35)
    {local state and tracker\\
        local prediction \(x_{n,t}\)};
    \draw[msg] (a1) -- (states);
    \draw[msg] (a2) -- (states);
    \draw[msg] (a3) -- (states);

    \node[target] (regret) at (0,-2.75)
    {network loss average \(\bar\ell_t\)\\
        compare with one comparator \(u\)};
    \draw[msg] (states) -- (regret);
\end{tikzpicture}
    \caption{System model for \algname{}\@.
        Agents make local predictions, observe local convex losses, and exchange compressed state differences over the communication graph.
        The analysis evaluates the averaged network loss against a single comparator \(u\).}\label{fig:system-model}
\end{figure}
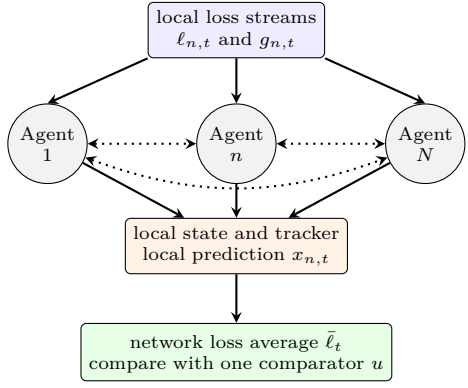
At Round $t$, Agent \(n\) calculates a local prediction \(x_{n,t}\) from its accumulated local  state, which summarizes its past feedback loss and previous gossip exchanges, using the coin-betting map defined in \zcref{subsec:betting-potentials}.
After observing a local subgradient, the agent updates a clean state and then performs \(q(t)\) compressed gossip substeps.
Each communication substep sends the compressed difference between the clean state and a tracker, rather than compressing the clean state itself.
This design preserves the clean network average while the tracker absorbs compression error through repeated difference updates.
The \emph{communication step size} \(\gamma\) is a communication parameter that should be chosen according to the graph and the compressor, but it does not depend on the loss sequence or comparator.
A larger \(\gamma\) results in faster consensus but can amplify the compression error, while a smaller \(\gamma\) reduces the compression error but slows consensus.
The gossip steps reduce clean-state disagreement, while the compressed difference updates keep the trackers close enough to the clean states for the Lyapunov contraction used in the proof.

The remaining subsections define the prediction map and the state update.

\subsection{Coin-Betting Potentials and Predictions}\label{subsec:betting-potentials}
The algorithm is parameter-free in the usual OCO sense: it does not choose a learning rate as a function of \(T\), or \(\norm u\).
It has a coin-betting endowment \(\varepsilon>0\), which fixes the initial wealth scale and communication parameters determined by the graph and compressor.
The prediction rule converts accumulated signed gradients into decisions through a coin-betting potential~\cite{orabona_coin_2016,orabona_modern_2023,Ortega_Jafarkhani_2025}.

\begin{definition}[Coin-Betting Potential]\label{def:potential}
    Let \(\varepsilon>0\).
    A family of functions \(F_t\), indexed by integers \(t\geq0\), with domains \(F_t:(-a_t,a_t)\to\R\) is a coin-betting potential if:
    \begin{enumerate}[(a)]
        \item \(F_0(0)=\varepsilon\) and \(a_t>t\).
        \item Each \(F_t\) is even, positive in \((-a_t,a_t)\), and \(\log F_t\) is convex in \((-a_t,a_t)\).
        \item For every \(t\geq1\), \(|x|\leq t-1\), and \(|c|\leq1\),
              \begin{equation}
                  \br{1+c\beta_t(x)}F_{t-1}(x)\geq F_t(x+c),
              \end{equation}
              where
              \begin{equation}
                  \beta_t(x)
                  =
                  \frac{F_t(x+1)-F_t(x-1)}
                  {F_t(x+1)+F_t(x-1)}.
              \end{equation}
    \end{enumerate}
    It is \emph{excellent} if, additionally, for every \(t\) and every \(0<R<a_t\), the square-root lift
    \begin{equation}
        z\mapsto F_t(\sqrt z)
        \label{eq:sqrt-lift}
    \end{equation}
    is convex on \([0,R^2]\).
    This extra condition is needed to lift the scalar coin betting to a Hilbert space.
\end{definition}

\begin{remark}
    In~\cite{orabona_coin_2016}, the excellence condition required \(F_t\) to be twice-differentiable and \(xF_t''(x)\geq F_t'(x)\).
    When potentials are twice-differentiable, both conditions are equivalent.
    However, we argue that the square-root formulation is less restrictive since it only requires convexity of~\eqref{eq:sqrt-lift}, not differentiability.
\end{remark}

The \emph{coin-betting function}~\cite{Ortega_Jafarkhani_2025} converts wealth into a one-dimensional wager through
\begin{equation}
    h_t(x)=\beta_t(x)F_{t-1}(x).
\end{equation}
For a Hilbert-space state \(G\in\Hil\), we use a radial version of \(h_t\), namely
\begin{equation}
    \operatorname{Bet}_t(G)=
    \begin{cases}
        0,                                & G=0,    \\
        \dfrac{h_t(\norm{G})}{\norm{G}}G, & G\neq0.
    \end{cases}
    \label{eq:betting-prediction}
\end{equation}
The potential formulas are only used on their reachable domain.
To keep the prediction input in that domain, the algorithm clips the norm of the state at radius \(t\) while preserving its direction:
\begin{equation}
    \clip_t(G)=
    \begin{cases}
        G,                    & \norm{G}\leq t, \\
        \dfrac{t}{\norm{G}}G, & \norm{G}>t.
    \end{cases}
    \label{eq:state-clip}
\end{equation}
Thus, \(\norm{\clip_t(G)}\leq t\) and \(\clip_t\) is the metric projection onto the closed ball of radius \(t\).
In the one-dimensional case, norm clipping is exactly the usual scalar clipping map \(\clip_t(z)=\min\{t,\max\{-t,z\}\}\).
Both maps are \(1\)-Lipschitz.
This clipping affects only the prediction query \(\operatorname{Bet}_{t+1}(\clip_t(G_{n,t}))\), the clean state used for gradient accumulation and compressed gossip remains the unclipped \(G_{n,t}\).
To give explicit examples of potentials, we introduce the two that we use throughout the paper:
\begin{align}
    F_t^{\mathrm{KT}}(x)
     & =
    \varepsilon\,
    \frac{2^t\Gamma\br{\frac{t+1}{2}+\frac{x}{2}}
        \Gamma\br{\frac{t+1}{2}-\frac{x}{2}}}
    {\pi\,t!},
    \label{eq:kt-potential} \\
    F_t^{\mathrm{SE}}(x)
     & =
    \frac{\varepsilon}{\sqrt{t+1}}
    \exp\br{\frac{x^2}{2(t+1)}}.
    \label{eq:shifted-exp-potential}
\end{align}
Here, \(\mathrm{KT}\) abbreviates Krichevsky-Trofimov and \(\mathrm{SE}\) abbreviates shifted-exponential.
The verification that these choices are excellent coin-betting potentials follows the standard KT and exponential-potential calculations\footnote{%
    The SE potential in \zcref{eq:shifted-exp-potential} is a shifted version of the original exponential potential~\cite{orabona_coin_2016}, which was defined as \(F_t(x)=\varepsilon\exp(x^2/(2t))\) for \(t\geq1\) and \(F_0(0)=\varepsilon\).
    However, the reader might notice that the unshifted exponential potential is not a coin-betting potential according to \zcref{def:potential}, since condition (c) fails at \(t=1\), \(x=0\), and any \(c\neq 0\).
    Indeed, \(\beta_1(0)=0\), so the required inequality would imply \(\varepsilon \geq \varepsilon \exp(c^2/2)\), which is false for \(c\neq0\).}, see~\cite{Ortega_Jafarkhani_2025,orabona_coin_2016}.

\subsection{\algname{} State and Update}\label{subsec:deco-ef}
The \algname{} update gossips compressed difference messages inside every communication substep.
Each agent stores two vectors:
\begin{itemize}
    \item \(G_{n,t}\): the true local accumulated state after the consensus correction;
    \item \(\widehat{G}_{n,t}\): the compressed tracker shared through gossip.
\end{itemize}
Both are initialized to zero.

At Round \(t\), Agent \(n\) applies the following update with \(\delta_{n,t}=-g_{n,t}\).
The communication step size \(0<\gamma\leq1\) controls how aggressively agents mix compressed tracker information, and the gossip schedule \(q:\mathbb N\to\mathbb N\) specifies how many compressed communication substeps are performed after the local gradient observation in Round \(t\).

\begin{algorithm}[t]
    \caption{\algname{}\@: one round at Agent \(n\)}\label{alg:deco-ef}
    \begin{algorithmic}[1]
        \Require{} State \((G_{n,t},\widehat{G}_{n,t})\), potential \(F\), gossip matrix \(W\), step size \(\gamma\), gossip count \(q(t)\).
        \State{} Predict \(x_{n,t}=\operatorname{Bet}_{t+1}(\clip_t(G_{n,t}))\).
        \State{} Observe \(g_{n,t}\in\partial \ell_{n,t}(x_{n,t})\) and set \(\delta_{n,t}=-g_{n,t}\).
        \State{} Local movement: \(Z_{n,t}^{0}=G_{n,t}+\delta_{n,t}\) and \(\widehat Z_{n,t}^{0}=\widehat G_{n,t}\).
        \For{\(k=0,\dots,q(t)-1\)}
        \State{} Difference payload: \(v_{n,t}^{k}=Z_{n,t}^{k}-\widehat Z_{n,t}^{k}\).
        \State{} Compress and broadcast to neighbors:
        \begin{equation*}
            m_{n,t}^{k}=\mathcal{C}(\omega_{n,t}^{k},v_{n,t}^{k}) .
        \end{equation*}
        \State{} Tracker update: \(\widehat Z_{n,t}^{k+1}=\widehat Z_{n,t}^{k}+m_{n,t}^{k}\).
        \State{} Relaxed gossip update:
        \begin{equation*}
            Z_{n,t}^{k+1}
            =
            Z_{n,t}^{k}
            +
            \gamma\sum_{j=1}^N W_{nj}
            \br{\widehat Z_{j,t}^{k+1}-\widehat Z_{n,t}^{k+1}}.
        \end{equation*}
        \EndFor%
        \State{} State update: \(G_{n,t+1}=Z_{n,t}^{q(t)}\) and \(\widehat G_{n,t+1}=\widehat Z_{n,t}^{q(t)}\).
    \end{algorithmic}
\end{algorithm}

The communication model is synchronous and reliable.
At Substep \(k\), Node \(n\) computes one compressed payload \(m_{n,t}^k\) and broadcasts the same payload to all neighbors with \(W_{jn}>0\).
Receivers update their local copy of Node \(n\)'s tracker by the same increment.
Equivalently, the matrix update is the shared-tracker view after all messages in that substep have arrived.
The form of the last line is important: since \(W\) is doubly stochastic,
\begin{equation*}
    \sum_{n=1}^N\sum_{j=1}^N
    W_{nj}\br{\widehat Z_{j,t}^{k+1}-\widehat Z_{n,t}^{k+1}}=0,
\end{equation*}
so every compressed gossip substep preserves the average of the clean state \(Z^k\).

\subsection{Notation Used in the Analysis}\label{subsec:notation}
Individual vectors in the decision space use the Hilbert norm.
For \(Z\in\Hil^N\), we define the product network norm as in \zcref{eq:product-network-norm}.
We write \(JZ\) for the consensus vector whose coordinates are all \(N^{-1}\sum_m Z_m\), and \(\Pperp Z=Z-JZ\) for centered disagreement.
The compression parameter \(\alphastoch\) is the mean-square residual coefficient from \zcref{def:compressor}.
We also use the spectral gap \(\mu=1-\rho\), the Lyapunov weight \(\chi=\mu/4\), and the expected compressed-gossip contraction rate \(\etastoch=1-\gamma\mu/2\).

\subsection{Operational Details}\label{sec:operational-details}
This subsection presents the implementation choices required by \algname{}.
\paragraph{Choosing the potential}\label{subsec:ops-potential}
Both potentials in~\zcref{eq:kt-potential,eq:shifted-exp-potential} satisfy the theory.
The shifted-exponential option is computationally simpler because it avoids Gamma-function evaluations and has the closed-form betting fraction \(\beta_t^{\mathrm{SE}}(x)=\tanh(x/(t+1))\).
\paragraph{Choosing the gossip schedule}\label{subsec:ops-gossip}
\zcref{thm:expected-network-regret} uses \(q(t)=\lceil ct\rceil\), with \(c\) chosen from the effective compressed-tracker rate in \zcref{eq:effective-tracker-rate}: KT requires \(c\geq -2\log2/\log\etastoch\), while shifted-exponential requires \(c\geq -3/(2\log\etastoch)\).
A larger \(c\) reduces disagreement faster but costs more communication.
\paragraph{Communication and contraction}\label{subsec:ops-gamma}
The relaxation \(\gamma\) controls the rate \(\etastoch=1-\gamma\mu/2\) in the compressed-tracker contraction.
With \(\mu=1-\rho\), the proof uses the sufficient condition
\begin{equation*}
    0<\gamma
    \leq
    \min\left\{
    1,\,
    \frac{1-\sqrt{\alphastoch}}{
        2\sqrt{\alphastoch}\left(1+\frac{4}{\mu}\right)+\frac{\mu}{2}
    }
    \right\}.
\end{equation*}
When \(\gamma\mu\) is small, \(-\log\etastoch\approx\gamma\mu/2\), so \(c^{\mathrm{KT}}\approx4\log2/(\gamma\mu)\) and \(c^{\mathrm{SE}}\approx3/(\gamma\mu)\).
Thus, within the stated sufficient condition, taking \(\gamma\) as large as the graph/compressor bound permits lowers the required linear-schedule constant.

\section{Main Results}\label{sec:main-results}
The analysis uses two reusable statements.
First, the network argument only needs a Hilbert-space coin-betting regret bound for the radial prediction map.
For a terminal potential, namely the potential function at time \(T\), we write its restricted Fenchel conjugate as
\begin{equation}
    F_T^*(\theta)
    =
    \sup_{0\leq s\leq T}\{\theta s-F_T(s)\}.
    \label{eq:restricted-fenchel}
\end{equation}
Square-root convex excellence (see~\zcref{def:potential}) turns the scalar wealth telescope into a Hilbert-space reward telescope, and the final comparator conversion is exactly the conjugate in~\eqref{eq:restricted-fenchel}.
The main theorems are stated in terms of this conjugate.
The elementary envelope, used only to turn the KT and shifted-exponential conjugates into closed-form corollaries, is discussed in \zcref{lem:quad-exp-fenchel} in Appendix~\ref{sec:appendix-auxiliary}.

\begin{theorem}[General network regret template]\label{thm:general-network-regret}
    Let \(\Hil\) be a real Hilbert space.
    Suppose a \algname{} run uses a coin-betting potential whose radial Hilbert-space regret is controlled by its restricted terminal conjugate.
    Let \(\bar G_t=N^{-1}\sum_n G_{n,t}\) and \(y_t=\operatorname{Bet}_{t+1}(\bar G_t)\).
    If its prediction disagreement satisfies
    \begin{equation*}
        \sum_{t=0}^{T-1}\frac1N\sum_{n=1}^N
        \norm{x_{n,t}-y_t}
        \leq C_{\mathrm{net}}\sqrt T,
    \end{equation*}
    then for all \(u\in\Hil\) and all \(T>0\),
    \begin{equation}
        R_T^{\mathrm{net}}(u)
        \leq
        F_T^*(\norm u)+\varepsilon+3C_{\mathrm{net}}\sqrt T .
        \label{eq:general-network-regret-summary}
    \end{equation}
\end{theorem}

Second, for both concrete potentials with a linear gossip schedule, \zcref{thm:expected-network-regret} uses the effective tracker rate
\begin{equation}
    \etastoch=1-\frac{\gamma\mu}{2}
    \label{eq:effective-tracker-rate}
\end{equation}
rather than the clean relaxed-gossip rate \(1-\gamma\mu\), which is the ideal zero-mean contraction rate without compression.
\begin{theorem}[Expected network regret]\label{thm:expected-network-regret}
    Let \algname{} run over a real Hilbert space \(\Hil\) with either the KT potential~\eqref{eq:kt-potential} or the shifted-exponential potential~\eqref{eq:shifted-exp-potential}.
    Assume \(W\) is a valid gossip matrix for a connected communication graph, and let \(\rho\in[0,1)\) be the graph-derived product-\(\ell_2\) zero-mean contraction factor. 
    Let the compressor, as in~\zcref{def:compressor}, have mean-square residual parameter \(\alphastoch\).
    Put
    \begin{equation*}
        \mu=1-\rho,\qquad \bar\sigma=\sqrt{\alphastoch}.
    \end{equation*}
    Here, \(\mu\) is the spectral gap.
    Let \(P\in\{\mathrm{KT},\mathrm{SE}\}\) denote the chosen potential and let \(F_T^*\) be its restricted terminal conjugate.
    Choose
    \begin{equation*}
        0<\gamma\leq
        \min\left\{
        1,\,
        \frac{1-\bar\sigma}{
            2\bar\sigma\left(1+\frac{4}{\mu}\right)+\frac{\mu}{2}}
        \right\},
    \end{equation*}
    define \(\etastoch=1-\gamma\mu/2\), and use a linear gossip schedule \(q(t)=\lceil ct\rceil\), where \(c\geq -\frac{2\log2}{\log\etastoch}\) for KT and \(c\geq -\frac{3}{2\log\etastoch}\) for the shifted-exponential potential.
    Put \(\chi=\mu/4\) and
    \begin{align*}
        C_{\mathrm{sched}}^{\mathrm{KT}}
         & =
        48\left((\varepsilon+1)\left(1+\frac2{\log2}\right)+1\right),
        \\
        C_{\mathrm{sched}}^{\mathrm{SE}}
         & =
        4e^{3/2}(2\varepsilon+1).
    \end{align*}
    With \(C_{\mathrm{sched}}^P\) selecting the matching constant above, set
    \begin{equation*}
        C_{\mathrm{net}}^P
        =
        (2+\chi)\frac{C_{\mathrm{sched}}^P}{2}\sqrt N,
        \qquad
        C_{\mathrm{reg}}^{P,\mathrm{stoch}}=3C_{\mathrm{net}}^P .
    \end{equation*}
    Then, for every comparator \(u\in\Hil\) and every \(T>0\),
    \begin{equation}
        \Exp R_T^{\mathrm{net}}(u)
        \leq
        F_T^*(\norm u)+\varepsilon+
        C_{\mathrm{reg}}^{P,\mathrm{stoch}}\sqrt T ,
        \label{eq:expected-network-regret-summary}
    \end{equation}
    where the expectation is over the compressor randomness.
    The constants are independent of \(u\) and \(T\).
\end{theorem}

The proof uses the conditional compressed-tracker contraction in \zcref{thm:expected-tracking-contraction}, the expected state-disagreement unrolling in \zcref{lem:expected-state-disagreement-unrolling}, and the potential-specific Lipschitz and schedule-tail estimates in \zcref{sec:analysis-kt,sec:analysis-shifted-exp}.

\begin{corollary}[Closed-form KT and shifted-exponential bounds]\label{cor:closed-form-network-regret}
    In the setting of \zcref{thm:expected-network-regret}, the KT potential satisfies
    \begin{align}
        \Exp R_T^{\mathrm{net}}(u)
         & \leq
        \norm u
        \sqrt{2(T+1)}
        \notag            \\
         & \quad {}\times
        \sqrt{
            \log\br{1+
                \frac{2e^2\pi T(T+1)\norm u^2}{\varepsilon^2}}
        }
        \notag            \\
         & \quad
        +\varepsilon\br{1-\frac{1}{e\sqrt{\pi T}}}
        +C_{\mathrm{reg}}^{\mathrm{KT},\mathrm{stoch}}\sqrt T,
        \label{eq:closed-form-kt-network-regret}
    \end{align}
    and the shifted-exponential potential satisfies
    \begin{align}
        \Exp R_T^{\mathrm{net}}(u)
         & \leq
        \norm u
        \sqrt{
            (T+1)
            \log\br{1+
                \frac{{(T+1)}^2\norm u^2}{\varepsilon^2}}
        }
        \notag   \\
         & \quad
        +\varepsilon\br{1-\frac{1}{\sqrt{T+1}}}
        +C_{\mathrm{reg}}^{\mathrm{SE},\mathrm{stoch}}\sqrt T.
        \label{eq:closed-form-se-network-regret}
    \end{align}
\end{corollary}

\begin{corollary}[Polylogarithmic comparator scale]\label{cor:tilde-network-regret}
    The two expected regret bounds in \zcref{eq:closed-form-kt-network-regret,eq:closed-form-se-network-regret} are \(\widetilde{O}(\max\{1,\norm u\}\sqrt T)\), where \(\widetilde{O}\) suppresses polylogarithmic factors in \(T\) and \(\norm u\).
\end{corollary}


\section{Analysis}\label{sec:analysis}
This section provides the proof idea for the main regret guarantees.
The auxiliary inequalities, Lyapunov algebra, and potential-specific estimates are proved in the appendices.
The main argument is as follows: reduce network regret to a centralized coin-betting term plus a prediction-disagreement term, bound the centralized term by Hilbert-space coin betting, and then use compressed-gossip tracking to show that the disagreement term is only \(O(\sqrt T)\) under the stated schedule.

\subsection{Regret Decomposition}\label{subsec:analysis-outline-regret}
Let
\begin{equation*}
    \bar g_t=\frac1N\sum_{n=1}^N g_{n,t},
    \qquad
    \bar G_t=\frac1N\sum_{n=1}^N G_{n,t},
\end{equation*}
and define the centralized prediction
\begin{equation*}
    y_t=\operatorname{Bet}_{t+1}(\bar G_t).
\end{equation*}
Since \(\bar G_t\) is the sum of \(t\) averaged negative subgradients of norm at most one, \(\norm{\bar G_t}\leq t\).
Hence, the centralized prediction lies in the reachable domain and no clipping is needed.
The average clean state is exactly the centralized cumulative negative gradient,
\begin{equation*}
    \bar G_t=\sum_{s=0}^{t-1}(-\bar g_s),
\end{equation*}
because every gossip correction has zero network average.
Thus, the network can be compared with a fictitious centralized coin-betting learner that sees the averaged gradients.
This centralized learner is the bridge between the parameter-free OCO analysis and the decentralized protocol: all remaining network effects enter only through the distance from \(x_{n,t}\) to \(y_t\).

Convexity and Lipschitzness give, round by round,
\begin{equation}
    \frac1N\sum_{n=1}^N\bar\ell_t(x_{n,t})-\bar\ell_t(u)
    \leq
    \langle\bar g_t,y_t-u\rangle
    +
    \frac3N\sum_{n=1}^N\norm{x_{n,t}-y_t}.
    \label{eq:main-regret-split-round}
\end{equation}
The constant \(3\) comes from comparing the global loss at \(x_{n,t}\) to the centralized prediction \(y_t\), replacing local subgradients by their network average, and using the \(1\)-Lipschitz assumption.
The first term is the regret of the centralized radial coin-betting learner.
The second term is the price of decentralization: it vanishes if all agents predict exactly at the average state.
Summing~\eqref{eq:main-regret-split-round} results in
\begin{equation}
    R_T^{\mathrm{net}}(u)
    \leq
    \sum_{t=0}^{T-1}\langle\bar g_t,y_t-u\rangle
    +
    3\sum_{t=0}^{T-1}\frac1N\sum_{n=1}^N\norm{x_{n,t}-y_t}.
    \label{eq:main-regret-split}
\end{equation}
This is the central proof split used throughout the paper.
The detailed one-round comparison is given in \zcref{lem:global-loss-centralized-comparison}.
We analyze these two terms in what follows.

\subsection{Centralized Coin Betting}\label{subsec:analysis-outline-coin}
We first handle the term that would remain even if every agent predicted at the network average.
Set \(c_t=-\bar g_t\).
Since each local loss is \(1\)-Lipschitz, \(\norm{c_t}\leq1\), and the centralized state above is \(\bar G_t=\sum_{s<t}c_s\).
Square-root convex excellence lets the scalar coin-betting potential lift radially to a Hilbert space: the one-dimensional wealth telescoping argument applies to \(\norm{\bar G_t}\), while Jensen's inequality controls the change in squared norm after adding the vector \(c_t\).
The terminal Fenchel conjugate~\eqref{eq:restricted-fenchel} then converts the wealth lower bound into
\begin{equation}
    \begin{aligned}
        \sum_{t=0}^{T-1}\langle\bar g_t,y_t-u\rangle
         & \leq
        F_T^*(\norm u)+\varepsilon .
    \end{aligned}
    \label{eq:main-centralized-bound}
\end{equation}
The scalar telescoping, Fenchel conversion, and Hilbert-space lift are proved in \zcref{sec:analysis-coin-betting}.
It remains to show that the actual local predictions stay close enough to this centralized trajectory.

\subsection{Compressed Tracking}\label{subsec:analysis-outline-tracking}
The disagreement term in~\eqref{eq:main-regret-split} is controlled first at the state level.
At a fixed outer round, suppress the round index and write \(Z^k\) and \(\widehat Z^k\) for the inner-loop clean state and tracker.
Formally, let us define the \emph{inner-loop} index \(k\) to run from \(0\) to \(q(t)-1\) for the \(q(t)\) gossip substeps at outer-loop time \(t\).
The inner gossip loop tracks two quantities: the clean-state disagreement \(D^k\) and the compression residual \(E^k\), defined by
\begin{equation}
    D^k=\norm{\Pperp Z^k}_{2,\Hil},
    \qquad
    E^k=\norm{Z^k-\widehat Z^k}_{2,\Hil}.
    \label{eq:tracking-components}
\end{equation}
With \(\mu=1-\rho\), \(\bar\sigma=\sqrt{\alphastoch}\), and \(\chi=\mu/4\), the Lyapunov quantity
\begin{equation*}
    V^k=D^k+\chi E^k
\end{equation*}
contracts in conditional expectation by the effective rate
\begin{equation*}
    \etastoch=1-\frac{\gamma\mu}{2}
\end{equation*}
whenever \(\gamma\) satisfies the step-size condition in \zcref{thm:expected-network-regret}.
The important point is that the rate is the compressed-tracker rate, not the cleaner relaxed-gossip rate one would obtain without residuals.
More precisely, conditionally on the history \(\mathcal F_{t,k}\) before gossip substep \(k\),
\begin{equation*}
    \Exp[V^{k+1}\mid\mathcal F_{t,k}]
    \leq
    \etastoch V^k .
\end{equation*}

Each outer round injects a bounded gradient perturbation, and then \(q(t)\) gossip substeps contract it.
Unrolling these injections in expectation gives
\begin{equation}
    \begin{aligned}
        \Exp\frac1N\sum_{n=1}^N\norm{G_{n,t}-\bar G_t}
         & \leq
        (2+\chi)\sqrt N
        \sum_{s<t}\etastoch^{Q(s,t)}, \\
        Q(s,t)
         & =
        \sum_{r=s}^{t-1}q(r).
    \end{aligned}
    \label{eq:main-state-disagreement}
\end{equation}
The Lyapunov contraction and unrolling details are collected in \zcref{sec:analysis-network-dynamics}.
The final step is to pass this state estimate through the nonlinear coin-betting prediction map.

\subsection{From State Disagreement to Regret}\label{subsec:analysis-outline-finish}
The prediction map is Lipschitz on the clipped ball of radius \(t\).
If \(L(t)\) denotes a valid envelope for the potential in use, then~\eqref{eq:main-state-disagreement} implies
\begin{equation*}
    \Exp\frac1N\sum_{n=1}^N\norm{x_{n,t}-y_t}
    \leq
    (2+\chi)\sqrt N\,L(t)
    \sum_{s<t}\etastoch^{Q(s,t)}.
\end{equation*}
For the KT potential, the envelope grows like \(2^t\log(t+2)/\sqrt{t+1}\).
For the shifted-exponential potential, it grows like \(\exp(t/2)/\sqrt{t+1}\).
For \(P\in\{\mathrm{KT},\mathrm{SE}\}\) denoting the chosen potential, the stated linear schedules make the gossip tail small enough that
\begin{equation}
    \Exp\sum_{t=0}^{T-1}\frac1N\sum_{n=1}^N\norm{x_{n,t}-y_t}
    \leq C_{\mathrm{net}}^P\sqrt T .
    \label{eq:main-prediction-disagreement}
\end{equation}
Taking expectations in~\eqref{eq:main-regret-split} and substituting~\eqref{eq:main-centralized-bound} and~\eqref{eq:main-prediction-disagreement} proves \zcref{thm:expected-network-regret}.
The centralized coin-betting term pays the comparator-adaptive part, while compressed tracking contributes only the stated network constant times \(\sqrt T\).

The KT and shifted-exponential analytic checks, including square-root convexity and the schedule-tail estimates, appear in \zcref{sec:analysis-kt,sec:analysis-shifted-exp}.

\section{Experiments}\label{sec:experiments}

\begin{figure*}[htbp]
    \centering
    \includegraphics[width=.85\textwidth]{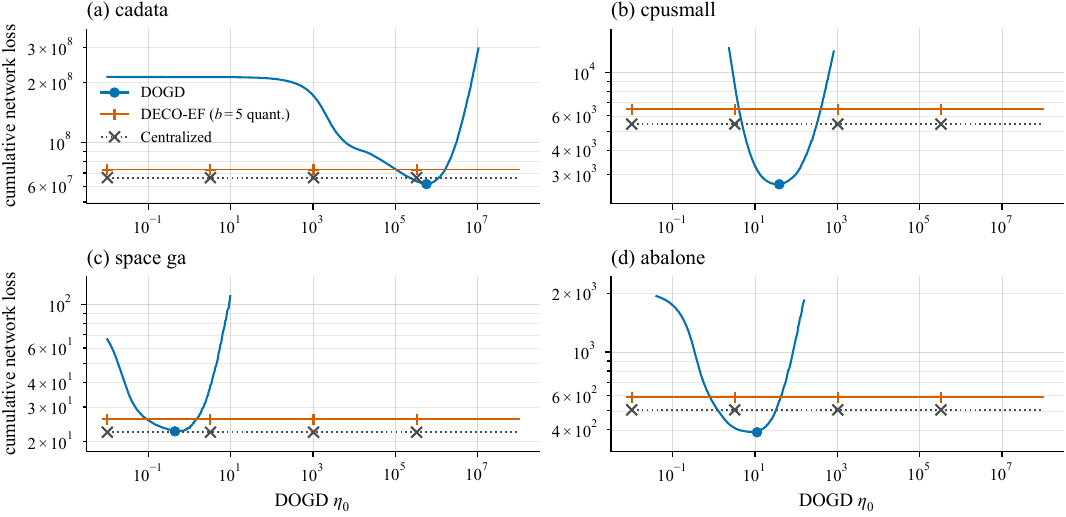}
    \caption{LIBSVM online-regression experiments, showing decentralized online gradient descent  (DOGD) tuning sensitivity on the ring graph against parameter-free methods.}\label{fig:libsvm-experiments}
\end{figure*}
\begin{figure*}[htbp]
    \centering
    \includegraphics[width=.85\textwidth]{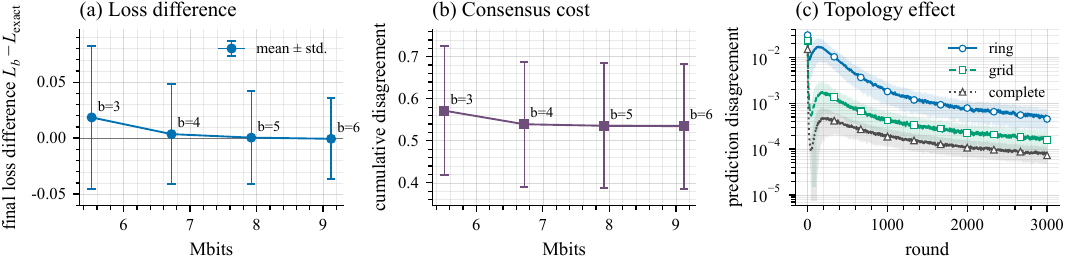}
    \caption{Synthetic quantized-communication diagnostics for DECO-EF\@.
    All decentralized runs use \(\gamma=1/2\), and the compressed runs use $b$-bit quantization.
    Each plotted point or curve is the average over one hundred independent synthetic streams, with error bars or shaded bands showing one standard deviation.
    Panels (a) and (b) use the complete graph with \(q(t)=1\), the topology that most clearly separates quantization resolution from slow graph mixing.
    Panel (a) reports the paired final cumulative network loss difference \(L_b-L_{\mathrm{exact}}\) against the no-quantization reference run with the same stream seed and relaxation.
    Panel (b) reports cumulative prediction disagreement against transmitted Mbits.
    Panel (c) fixes \(b=3\) and \(q(t)=1\), and compares the disagreement on the ring, \(4\times5\) toroidal grid, and complete graph.}\label{fig:synthetic-compression}
\end{figure*}

We evaluate two complementary effects: LIBSVM~\cite{LIBSVM} streams test DOGD tuning sensitivity against parameter-free coin betting, while a controlled heterogeneous synthetic stream isolates quantized-communication effects on loss, bit budget, and consensus.
Following the previous DECO evaluation protocol, we use the absolute linear prediction loss \(\ell_{n,t}(x)=|\langle x,z_{n,t}\rangle-y_{n,t}|\) with \(\norm{z_{n,t}}\leq1\), so that every subgradient has a norm at most one.
The real-data tasks use row-scaled LIBSVM regression datasets \texttt{cadata}, \texttt{cpusmall}, \texttt{space\_ga}, and \texttt{abalone}, assigned to \(N=20\) agents in round-robin order on a ring with \(q(t)=1\).
The DECO-EF curves use the shifted-exponential potential with \(b=5\)-bit uniform scalar quantization.
The centralized coin-betting reference uses the KT potential, exact communication and the complete graph.
Thus, it is a non-decentralized parameter-free reference rather than a communication-limited method.
All decentralized coin-betting runs use \(\gamma=1/2\).

Since our topologies are regular, we use the uniform gossip matrix: \(W_{ij}=1/\deg\) when \(j\) is in the neighborhood of \(i\), and \(W_{ij}=0\) otherwise, with \(\deg\) including the self-loop.
For synthetic compression diagnostics, all runs use the shifted-exponential DECO-EF\@.
Communication costs count raw dense uniform-scalar quantization bits, one 32-bit scale plus \(b\) bits per coordinate, reported as the total payload emitted by all agents rather than multiplied by graph edges.
For fixed data, method, compressor, and \(q(t)\), the Mbits value is therefore the same across topologies.

The synthetic stream has a dimension of \(20\) and \(T=3000\) rounds.
For each seed \(r\in\{0,\ldots,99\}\), we draw and normalize \(u^\star\sim\mathcal N(0,I_{20})\), then draw shared base vectors \(b_t\sim\mathcal N(0,I_{20})\) and agent centers \(\mu_n\sim\mathcal N(0,2.0^2I_{20})\).
We set \(z_{n,t}=a_{n,t}\min\{1,1/\norm{a_{n,t}}\}\) with \(a_{n,t}=b_t+\mu_n\).
The labels are \(y_{n,t}=\langle u^\star,z_{n,t}\rangle+\nu_{n,t}\), where \(\nu_{n,t}\sim\mathcal N(0,0.03^2)\) is independent across agents and rounds.
This preserves a common regression signal while making local streams heterogeneous enough for consensus effects to be visible.

\zcref{fig:libsvm-experiments} shows DOGD tuning sensitivity across the four ordered LIBSVM streams.
The DOGD curves confirm that the best learning rate varies widely across datasets, while DECO-EF and the centralized coin-betting baseline show consistent performance.
\zcref{fig:synthetic-compression} illustrates the communication and topology aspects.
Coarse quantization has the largest observable effect on final cumulative loss variance, \(b=3\) leaves more cumulative disagreement than \(b=4,5,6\), and the grid and complete graph remove prediction disagreement much faster than the ring.

\section{Conclusion}\label{sec:conclusion}

This paper introduced \algname{}, a decentralized online learning family that avoids OCO learning-rate tuning while communicating compressed state differences.
Under a linear gossip schedule, \algname{} achieves the expected \(\widetilde{O}(\max\{1,\norm u\}\sqrt T)\) regret rate of centralized parameter-free methods, up to a network/compressor constant independent of \(T\) and \(\norm u\).
Synthetic and real-data experiments illustrate the sensitivity of DOGD tuning and the communication/disagreement tradeoffs of \algname{}.
A natural next target is the gap between the linear-schedule guarantee and the attractive constant communication regime \(q(t)=O(1)\).

\appendices%
\zcsetup{
    countertype = {
            section       = appendix,
            subsection    = appendix,
            subsubsection = appendix,
            paragraph     = appendix,
            subparagraph  = appendix
        }
}
\section{Proof Details}\label{sec:appendix-auxiliary}%
\label{sec:analysis-coin-betting}%
\label{sec:analysis-network-dynamics}%
\label{sec:analysis-regret-decomposition}%
\label{sec:analysis-kt}%
\label{sec:analysis-shifted-exp}%
\label{sec:prediction-disagreement-details}
This appendix collects the coin-betting reductions, compressed-tracking estimates, regret decomposition, potential-specific tail calculations, and final prediction-disagreement bounds used in \zcref{sec:analysis}.
For the closed-form corollaries, we use the following quadratic-exponential envelope:
\begin{equation}
    \Phi_{\tau,\kappa}(\theta)
    =
    \theta\sqrt{\tau\log\br{1+\frac{\tau\theta^2}{\kappa^2}}}
    -\kappa .
    \label{eq:quad-exp-fenchel-envelope}
\end{equation}
For the two concrete potentials, set
\begin{align*}
    \Psi_T^{\mathrm{KT}}(\theta)
     & =
    \Phi_{2(T+1),\,\varepsilon/(e\sqrt{\pi T})}(\theta),
    \\
    \Psi_T^{\mathrm{SE}}(\theta)
     & =
    \Phi_{T+1,\,\varepsilon/\sqrt{T+1}}(\theta).
\end{align*}

\paragraph{Coin-betting reductions}
Let \(P=\{F_t\}\) be a coin-betting potential and \(h_t(x)=\beta_t(x)F_{t-1}(x)\).
For a scalar coin sequence \(c_t\in[-1,1]\), define \(S_t=\sum_{s=0}^{t-1}c_s\).

\begin{lemma}[Quadratic-exponential Fenchel envelope]\label{lem:quad-exp-fenchel}
    Let \(\tau,\kappa>0\) and \(x,y\geq0\).
    Then,
    \begin{equation*}
        xy-\kappa\exp\br{\frac{x^2}{2\tau}}
        \leq
        y\sqrt{\tau\log\br{1+\frac{\tau y^2}{\kappa^2}}}-\kappa .
    \end{equation*}
\end{lemma}
\begin{IEEEproof}
    Set
    \begin{equation*}
        a=\frac{y\sqrt\tau}{\kappa},\qquad
        v=\frac{x}{\sqrt\tau},\qquad
        r=\sqrt{\log(1+a^2)} .
    \end{equation*}
    It suffices to prove \(av-\exp(v^2/2)\leq ar-1\).
    If \(v\leq r\), then \(av\leq ar\) and \(\exp(v^2/2)\geq1\).
    If \(v>r\), write \(d=v-r\geq0\).
    The inequality \(q\log q\geq q-1\) for \(q\geq1\), applied to \(q=1+a^2\), gives \(a\leq r\exp(r^2/2)\).
    Together with \(\exp z\geq1+z\) and \(r^2/2+rd\leq v^2/2\), this yields \(\exp(v^2/2)\geq1+ad\), which is equivalent to the desired bound.
\end{IEEEproof}

\begin{lemma}[Concrete terminal conjugate envelopes]\label{lem:concrete-conjugate-envelopes}
    For every integer \(T\geq1\) and every \(\theta\geq0\), the KT and shifted-exponential terminal conjugates satisfy
    \begin{equation*}
        {(F_T^{\mathrm{KT}})}^*(\theta)\leq \Psi_T^{\mathrm{KT}}(\theta),
        \qquad
        {(F_T^{\mathrm{SE}})}^*(\theta)\leq \Psi_T^{\mathrm{SE}}(\theta).
    \end{equation*}
\end{lemma}
\begin{IEEEproof}
    For the shifted-exponential potential, the bound follows directly from \zcref{lem:quad-exp-fenchel} with
    \(\tau=T+1\) and \(\kappa=\varepsilon/\sqrt{T+1}\).

    For the KT potential, one can show using standard Gamma properties, that on \(0\leq s\leq T\),
    \begin{equation*}
        F_T^{\mathrm{KT}}(s)
        \geq
        \frac{\varepsilon}{e\sqrt{\pi T}}
        \exp\!\left(\frac{s^2}{4(T+1)}\right).
    \end{equation*}
    Applying \zcref{lem:quad-exp-fenchel} with
    \(\tau=2(T+1)\) and \(\kappa=\varepsilon/(e\sqrt{\pi T})\) gives the KT envelope.
\end{IEEEproof}

\begin{theorem}[One-dimensional coin-betting regret]\label{thm:oned-coin}
    Suppose the restricted conjugate~\eqref{eq:restricted-fenchel} is finite.
    Then, for every \(\theta\geq0\) and every \(T\geq1\),
    \begin{equation}
        \begin{aligned}
            \theta S_T-\sum_{t=0}^{T-1}c_t h_{t+1}(S_t)
             & \leq
            F_T^*(\theta)+\varepsilon .
        \end{aligned}
        \label{eq:oned-betting-regret}
    \end{equation}
\end{theorem}
See~\cite{Ortega_Jafarkhani_2025} for the proof of~\zcref{thm:oned-coin}.

The following is the only place where excellence enters the proof.
Once this radial bound is available, the decentralized and scalar reductions use only the ordinary coin-betting potential and the stated Lipschitz estimates.
\begin{theorem}[Radial Hilbert-space coin-betting regret]\label{thm:hilbert-coin}
    Let \(\Hil\) be a real Hilbert space and let \(c_t\in\Hil\) satisfy \(\norm{c_t}\leq1\).
    Put \(S_t=\sum_{s=0}^{t-1}c_s\) and predict \(y_t=\operatorname{Bet}_{t+1}(S_t)\) using~\eqref{eq:betting-prediction}.
    If \(P\) is excellent in the square-root sense and has finite restricted conjugates, then for every \(u\in\Hil\) and every integer \(T\geq1\),
    \begin{equation*}
        \begin{aligned}
            \sum_{t=0}^{T-1}
            \br{\langle u,c_t\rangle-\langle c_t,y_t\rangle}
             & \leq
            F_T^*(\norm u)+\varepsilon .
        \end{aligned}
    \end{equation*}
\end{theorem}
\begin{IEEEproof}
    The lifted betting map is radial:
    \begin{equation*}
        \operatorname{Bet}_t(G)=
        \frac{h_t(\norm{G})}{\norm{G}}G
        \qquad(G\neq0).
    \end{equation*}
    The square-root convexity condition is used only to compare squared norms.
    When \(S_t\neq0\), set \(r=\norm{S_t}\) and \(\zeta=\langle c_t,S_t\rangle/r\).
    Then, \(|\zeta|\leq1\), and
    \begin{equation*}
        \norm{S_t+c_t}^2
        \leq
        \frac{1+\zeta}{2}{(r+1)}^2
        +
        \frac{1-\zeta}{2}{(r-1)}^2.
    \end{equation*}
    Convexity of \(z\mapsto F_{t+1}(\sqrt z)\) and the scalar coin-betting inequality at \(r\) therefore give the one-step reward inequality
    \begin{equation*}
        \langle c_t,\operatorname{Bet}_{t+1}(S_t)\rangle
        \geq
        F_{t+1}(\norm{S_t+c_t})-F_t(\norm{S_t}).
    \end{equation*}
    The case \(S_t=0\) follows directly from the same scalar inequality and evenness of the potential.
    Summing over \(t\) telescopes results in a cumulative reward that is at least \(F_T(\norm{S_T})-\varepsilon\).
    Since \(\norm{S_T}\leq T\), the same Fenchel-conjugate argument as in the scalar theorem, now with \(\theta=\norm{u}\) and \(\langle u,S_T\rangle\leq\norm{u}\norm{S_T}\), yields the bound.
\end{IEEEproof}

\paragraph{Network dynamics}
We next record the deterministic tracking estimates used by \algname{}\@.

\begin{lemma}[Average state tracks the average cumulative gradient]\label{lem:average-tracks}
    For every run and every \(t\),
    \begin{equation*}
        \frac1N\sum_{n=1}^N G_{n,t}
        =
        \sum_{s=0}^{t-1}\frac1N\sum_{n=1}^N (-g_{n,s}).
    \end{equation*}
\end{lemma}
\begin{IEEEproof}
    At \(t=0\) both sides are zero.
    Suppose the identity holds at time \(t\).
    The local movement forms \(G_t+\delta_t\), where \(\delta_{n,t}=-g_{n,t}\).
    Each compressed gossip substep changes \(G\) by
    \begin{equation*}
        \gamma\sum_j W_{nj}(\widehat Z_j-\widehat Z_n).
    \end{equation*}
    Summing this correction over \(n\) gives zero by the row and column stochasticity of \(W\).
    Hence, every compressed gossip substep preserves the network average of \(G\).
    Thus, the average at time \(t+1\) is the previous average plus \(\frac1N\sum_n(-g_{n,t})\), proving the induction step.
\end{IEEEproof}

\begin{lemma}[Lyapunov tail unrolling]\label{lem:lyapunov-tail-unroll}
    Let \(q:\mathbb N\to\mathbb N\), \(P\geq0\), and \(\eta\geq0\).
    Define \(Q(s,t)=\sum_{r=s}^{t-1}q(r)\).
    If \(a_0=0\) and
    \begin{equation*}
        a_{t+1}\leq \eta^{q(t)}(a_t+P)
    \end{equation*}
    for every \(t\), then
    \begin{equation*}
        a_t\leq P\sum_{s<t}\eta^{Q(s,t)}
    \end{equation*}
    for every \(t\).
\end{lemma}
\begin{IEEEproof}
    The claim is immediate at \(t=0\).
    For the induction step, use \(Q(s,t+1)=Q(s,t)+q(t)\) for \(s<t\) and \(Q(t,t+1)=q(t)\):
    \begin{equation*}
        \begin{aligned}
            a_{t+1}
             & \leq
            \eta^{q(t)}
            \left(P\sum_{s<t}\eta^{Q(s,t)}+P\right) \\
             & =
            P\sum_{s<t}\eta^{Q(s,t+1)}
            +P\eta^{Q(t,t+1)}
            =
            P\sum_{s<t+1}\eta^{Q(s,t+1)} .
        \end{aligned}
    \end{equation*}
\end{IEEEproof}

\begin{lemma}[Raw compressed-gossip  recurrences]\label{lem:raw-tracking-recurrences}
    Let \(W\) be a valid gossip matrix for a connected communication graph, and let \(L=I-W\).
    Suppose the Hilbert lift of \(W\) contracts the zero-mean subspace by \(0\leq\rho<1\), and put \(\mu=1-\rho\).
    Fix one compressed gossip substep with \(0<\gamma\leq1\).
    Define
    \begin{equation*}
        e^k=Z^k-\widehat Z^k,
        \qquad
        m^k=\mathcal C(\omega^k,e^k),
        \qquad
        r^k=e^k-m^k,
    \end{equation*}
    and let
    \begin{equation*}
        D^k=\norm{\Pperp Z^k}_{2,\Hil},
        \qquad
        E^k=\norm{e^k}_{2,\Hil}.
    \end{equation*}
    Then the inner-loop update satisfies
    \begin{align}
        D^{k+1}
         & \leq
        (1-\gamma\mu)D^k+2\gamma\norm{r^k}_{2,\Hil},
        \label{eq:raw-D-recursion} \\
        E^{k+1}
         & \leq
        2\gamma D^k+(1+2\gamma)\norm{r^k}_{2,\Hil}.
        \label{eq:raw-E-recursion}
    \end{align}
\end{lemma}
\begin{IEEEproof}
    Since \(W\) is doubly stochastic, \(LJ=0\), so \(LZ=L\Pperp Z\).
    Product-norm nonexpansiveness of \(W\) also gives \(\norm{LU}_{2,\Hil}\le2\norm{U}_{2,\Hil}\).
    The tracker update is
    \begin{equation*}
        \widehat Z^{k+1}=\widehat Z^k+m^k=Z^k-r^k,
    \end{equation*}
    hence
    \begin{equation*}
        Z^{k+1}
        =
        Z^k-\gamma L\widehat Z^{k+1}
        =
        (I-\gamma L)Z^k+\gamma Lr^k .
    \end{equation*}
    On zero-mean vectors, \(I-\gamma L=(1-\gamma)I+\gamma W\) contracts by \(1-\gamma+\gamma\rho=1-\gamma\mu\), and \(\norm{Lr^k}_{2,\Hil}\le2\norm{r^k}_{2,\Hil}\), which gives \zcref{eq:raw-D-recursion}.
    For the residual,
    \begin{equation*}
        Z^{k+1}-\widehat Z^{k+1}
        =
        r^k-\gamma L(Z^k-r^k).
    \end{equation*}
    Using \(LZ^k=L\Pperp Z^k\) and \(\norm{LU}_{2,\Hil}\le2\norm{U}_{2,\Hil}\) gives \zcref{eq:raw-E-recursion}.
\end{IEEEproof}

\begin{theorem}[Expected compressed tracker contraction]\label{thm:expected-tracking-contraction}
    Let \(W\) be a valid gossip matrix for a connected communication graph.
    Let \(0\leq\rho<1\) be the resulting product-\(\ell_2\) zero-mean contraction factor.
    Its Hilbert lift satisfies 
    \begin{equation*}
        \begin{aligned}
            \norm{WZ}_{2,\Hil}
             & \leq\norm{Z}_{2,\Hil},     \\
            \norm{WZ}_{2,\Hil}
             & \leq \rho\norm{Z}_{2,\Hil}
            \quad\text{for zero-mean }Z,
            \qquad 0\leq\rho<1 .
        \end{aligned}
    \end{equation*}
    Put
    \begin{equation*}
        \mu=1-\rho,\qquad \bar\sigma=\sqrt{\alphastoch}.
    \end{equation*}
    Assume
    \begin{equation*}
        0<\gamma\leq
        \min\left\{
        1,\,
        \frac{1-\bar\sigma}{
            2\bar\sigma\left(1+\frac{4}{\mu}\right)+\frac{\mu}{2}}
        \right\}.
    \end{equation*}
    Let
    \begin{equation*}
        \chi=\frac{\mu}{4},
        \qquad
        \etastoch=1-\frac{\gamma\mu}{2}.
    \end{equation*}
    For one compressed gossip substep, use \(D^k\) and \(E^k\) from \zcref{eq:tracking-components} and define
    \begin{equation}
        V^k=D^k+\chi E^k.
        \label{eq:compressed-lyapunov}
    \end{equation}
    Suppose the compressor randomness is fresh conditionally on the current payloads, so that
    \begin{equation*}
        \Exp\!\left[\norm{r^k}_{2,\Hil}\mid Z^k,\widehat Z^k\right]\leq
        \sqrt{\alphastoch}\norm{e^k}_{2,\Hil}.
    \end{equation*}
    Then one step satisfies
    \begin{equation*}
        \Exp\!\left[V^{k+1}\mid Z^k,\widehat Z^k\right]
        \leq
        \etastoch V^k .
    \end{equation*}
\end{theorem}
\begin{IEEEproof}
    Apply the raw recurrences in \zcref{lem:raw-tracking-recurrences} to the current inner-loop substep.
    The mean-square compressor condition gives the stated vector residual estimate.
    Indeed, conditionally on \(e^k\),
    \begin{equation*}
        \Exp\!\left[\norm{r^k}_{2,\Hil}^2\mid Z^k,\widehat Z^k\right]
        =
        \sum_n\Exp\norm{r_n^k}^2
        \leq
        \alphastoch\norm{e^k}_{2,\Hil}^2 .
    \end{equation*}
    Jensen's inequality therefore gives
    \(\Exp[\norm{r^k}_{2,\Hil}\mid Z^k,\widehat Z^k]\leq\bar\sigma E^k\).
    Taking conditional expectations in \zcref{eq:raw-D-recursion,eq:raw-E-recursion},
    \begin{align}
        \Exp[D^{k+1}\mid Z^k,\widehat Z^k]
         & \leq
        (1-\gamma\mu)D^k+2\gamma\bar\sigma E^k,
        \label{eq:expected-D-recursion} \\
        \Exp[E^{k+1}\mid Z^k,\widehat Z^k]
         & \leq
        2\gamma D^k+\bar\sigma(1+2\gamma)E^k.
        \label{eq:expected-E-recursion}
    \end{align}
    Combining these recurrences with \(V^k=D^k+\chi E^k\),
    \begin{align*}
        \Exp[V^{k+1}\mid Z^k,\widehat Z^k]
         & \leq
        (1-\gamma\mu+2\chi\gamma)D^k \\
         & \quad+
        (2\gamma\bar\sigma+\chi\bar\sigma(1+2\gamma))E^k .
    \end{align*}
    With \(\chi=\mu/4\), the \(D^k\) coefficient becomes \(1-\gamma\mu/2=\etastoch\).
    The stated bound on \(\gamma\) is exactly the rearranged condition
    \begin{equation*}
        2\gamma\bar\sigma+\chi\bar\sigma(1+2\gamma)
        \leq
        \chi\etastoch .
    \end{equation*}
    Thus the \(E^k\) coefficient is at most \(\chi\etastoch\), and the displayed Lyapunov contraction follows.
\end{IEEEproof}

\begin{lemma}[Expected one-round Lyapunov recurrence]\label{lem:expected-lyapunov-recurrence}
    Under the hypotheses of \zcref{thm:expected-tracking-contraction}, let \(\mathcal F_t\) be the history before the compressor randomness in Round \(t\), and let \(V_t\) be the outer-round Lyapunov quantity.
    Then,
    \begin{equation*}
        \Exp[V_{t+1}\mid\mathcal F_t]
        \leq
        \etastoch^{q(t)}
        \left(V_t+(2+\chi)\sqrt N\right).
    \end{equation*}
\end{lemma}
\begin{IEEEproof}
    At the start of Round \(t\),
    \(Z_t^0=G_t+\delta_t\), \(\widehat Z_t^0=\widehat G_t\), and \(\delta_{n,t}=-g_{n,t}\).
    Since the losses are \(1\)-Lipschitz, \(\norm{\delta_t}_{2,\Hil}\le\sqrt N\).
    The triangle inequality increases \(\norm{\Pperp G_t}_{2,\Hil}\) by at most \(2\sqrt N\) and \(\norm{G_t-\widehat G_t}_{2,\Hil}\) by at most \(\sqrt N\), so \(V_t^0\le V_t+(2+\chi)\sqrt N\).
    Applying \zcref{thm:expected-tracking-contraction} to each of the \(q(t)\) fresh inner steps and using the tower property gives the claim.
\end{IEEEproof}

\begin{lemma}[Expected compressed state-disagreement unrolling]\label{lem:expected-state-disagreement-unrolling}
    Under the expected hypotheses, with \(Q(s,t)=\sum_{r=s}^{t-1}q(r)\),
    \begin{equation*}
        \Exp\frac1N\sum_{n=1}^N\norm{G_{n,t}-\bar G_t}
        \leq
        (2+\chi)\sqrt N
        \sum_{s<t}\etastoch^{Q(s,t)} .
    \end{equation*}
\end{lemma}
\begin{IEEEproof}
    Let \(a_t=\Exp V_t\) and \(P=(2+\chi)\sqrt N\).
    Taking expectation in \zcref{lem:expected-lyapunov-recurrence} gives
    \begin{equation*}
        a_{t+1}\leq \etastoch^{q(t)}(a_t+P),
        \qquad a_0=0.
    \end{equation*}
    The unrolling in \zcref{lem:lyapunov-tail-unroll} applies to this scalar recurrence.
    Finally,
    \begin{equation*}
        \frac1N\sum_n\norm{G_{n,t}-\bar G_t}
        \leq
        \norm{\Pperp G_t}_{2,\Hil}
        \leq V_t
    \end{equation*}
    for every realization, so the same bound holds after expectation.
\end{IEEEproof}

\paragraph{Disagreement and network regret}
The regret proof separates the centralized radial coin-betting term from the disagreement induced by nonidentical agent predictions.
The regret assembly uses the global-loss comparison in \zcref{lem:global-loss-centralized-comparison} directly.
The following summation lemma is the shared tail estimate used by both potentials.

\begin{lemma}[Square-root disagreement criterion]\label{lem:sqrt-disagreement}
    Let \(L(t)\) be a bound on the prediction-map Lipschitz constant at time \(t\), and define
    \begin{equation*}
        Q(s,t)=\sum_{r=s}^{t-1}q(r).
    \end{equation*}
    If, for all \(t\geq1\),
    \begin{equation*}
        L(t)\sum_{s<t}\eta^{Q(s,t)}
        \leq
        \frac{C_{\mathrm{lip}}}{\sqrt t},
    \end{equation*}
    then
    \begin{equation*}
        2\sqrt N\sum_{t=1}^T
        L(t)\sum_{s<t}\eta^{Q(s,t)}
        \leq
        4C_{\mathrm{lip}}\sqrt N\sqrt T.
    \end{equation*}
\end{lemma}
\begin{IEEEproof}
    Apply the assumed pointwise bound term by term:
    \begin{equation*}
        2\sqrt N\sum_{t=1}^T
        L(t)\sum_{s<t}\eta^{Q(s,t)}
        \leq
        2C_{\mathrm{lip}}\sqrt N
        \sum_{t=1}^T\frac1{\sqrt t}.
    \end{equation*}
    The integral comparison
    \(\sum_{t=1}^T t^{-1/2}\leq2\sqrt T\)
    gives the result.
\end{IEEEproof}

\begin{theorem}[General network regret template]\label{thm:general-regret-proof}
    Let
    \begin{equation*}
        \bar g_t=\frac1N\sum_{n=1}^N g_{n,t},
        \quad
        \bar G_t=\frac1N\sum_{n=1}^N G_{n,t},
        \quad
        y_t=\operatorname{Bet}_{t+1}(\bar G_t).
    \end{equation*}
    Suppose the centralized radial coin-betting term satisfies
    \begin{equation*}
        \begin{aligned}
            \sum_{t=0}^{T-1}
            \langle \bar g_t,y_t-u\rangle
             & \leq
            F_T^*(\norm u)+\varepsilon
        \end{aligned}
    \end{equation*}
    for every \(u\) and every integer \(T\geq1\).
    Suppose also that
    \begin{equation*}
        \sum_{t=0}^{T-1}\frac1N\sum_{n=1}^N\norm{x_{n,t}-y_t}
        \leq C_{\mathrm{net}}\sqrt T .
    \end{equation*}
    Then, for every \(u\in\Hil\) and every integer \(T\geq1\),
    \begin{equation*}
        \begin{aligned}
            R_T^{\mathrm{net}}(u)
             & \leq
            F_T^*(\norm u)+\varepsilon
            +3C_{\mathrm{net}}\sqrt T .
        \end{aligned}
    \end{equation*}
\end{theorem}
\begin{IEEEproof}
    \zcref{lem:global-loss-centralized-comparison} gives, for each round,
    \begin{equation*}
        \frac1N\sum_{n=1}^N\bar\ell_t(x_{n,t})-\bar\ell_t(u)
        \leq
        \langle\bar g_t,y_t-u\rangle
        +\frac3N\sum_{n=1}^N\norm{x_{n,t}-y_t}.
    \end{equation*}
    Summing over \(t\) and applying the two assumed bounds gives the desired exact conjugate plus disagreement bound.
\end{IEEEproof}

\begin{lemma}[Radial prediction Lipschitz envelope]\label{lem:radial-lipschitz-envelope}
    Let \(H(G)=0\) for \(G=0\) and \(H(G)=h(\norm G)G/\norm G\) otherwise.
    Suppose that on \(0\leq r\leq R\) the scalar profile satisfies
    \begin{equation*}
        |h'(r)|\leq A,
        \qquad
        |h(r)|/r\leq A\quad(r>0).
    \end{equation*}
    Then \(H\) is \(A\)-Lipschitz on the closed ball \(\{G:\norm G\leq R\}\).
    Consequently \(H\circ\clip_R\) is also \(A\)-Lipschitz on all of \(\Hil\).
\end{lemma}
\begin{IEEEproof}
    Away from the origin, the derivative of the radial map has radial eigenvalue \(h'(r)\) and tangential eigenvalues \(h(r)/r\).
    The expressions therefore bound the operator norm of the derivative by \(A\).
    Continuity at the origin follows from \(|h(r)|\leq Ar\).
    Integrating the derivative along the line segment between two points in the ball gives the Lipschitz bound.
    The clipping map is the metric projection onto a closed Hilbert ball, hence is \(1\)-Lipschitz, so composition with \(\clip_R\) preserves the same envelope.
\end{IEEEproof}

\paragraph{Potential Specializations}
For the KT potential, the regret proof uses the following auxiliary facts.



\begin{lemma}[KT prediction Lipschitz envelope]\label{lem:kt-prediction-lipschitz}
    Let \(B_t^{\mathrm{KT}}(G)= \operatorname{Bet}_{t+1}^{\mathrm{KT}}(\clip_t(G))\).
    For every \(t\geq1\),
    \begin{equation*}
        \norm{B_t^{\mathrm{KT}}(a)-B_t^{\mathrm{KT}}(b)}
        \leq
        C_{\star}^{\mathrm{KT}}(\varepsilon)
        \frac{2^t\log(t+2)}{\sqrt{t+1}}\norm{a-b},
    \end{equation*}
    where \(C_{\star}^{\mathrm{KT}}(\varepsilon)= (\varepsilon+1)(1+2/\log2)+1\).
\end{lemma}
\begin{IEEEproof}
    On \(0<r\le t\), \(\beta_{t+1}^{\mathrm{KT}}(r)=r/(t+1)\), the Gamma recurrence, and \(\psi(x+1)-\psi(x)\le1/x\) give
    \begin{equation*}
        |(h_{t+1}^{\mathrm{KT}})'(r)|,\ |h_{t+1}^{\mathrm{KT}}(r)|/r
        \le C_{\star}^{\mathrm{KT}}(\varepsilon)\frac{2^t\log(t+2)}{\sqrt{t+1}}.
    \end{equation*}
    These are exactly the radial and tangential derivative bounds needed by \zcref{lem:radial-lipschitz-envelope}.
    Since \(\clip_t\) is nonexpansive, the clipped prediction map has the same envelope.
\end{IEEEproof}

\begin{lemma}[Linear KT gossip schedule]\label{lem:kt-linear-disagreement}
    Fix a rate \(0<\eta<1\) and endowment \(\varepsilon>0\).
    Define
    \begin{equation*}
        C_{\star}^{\mathrm{KT}}(\varepsilon)
        =
        (\varepsilon+1)\br{1+\frac{2}{\log2}}+1.
    \end{equation*}
    If \(q(t)=\lceil ct\rceil\) and \(c\geq -2\log2/\log\eta\), then the KT disagreement sum is \(O(\sqrt T)\):
    \begin{equation*}
        \begin{aligned}
             & 2\sqrt N\sum_{t=1}^T
            \left(
            C_{\star}^{\mathrm{KT}}(\varepsilon)
            \frac{2^t\log(t+2)}{\sqrt{t+1}}
            \right)
            \sum_{s<t}\eta^{Q(s,t)}
            \\
             & \qquad\leq
            48C_{\star}^{\mathrm{KT}}(\varepsilon)\sqrt N\sqrt T
        \end{aligned}
    \end{equation*}
    for every \(T\geq1\).
\end{lemma}
\begin{IEEEproof}
    \zcref{lem:kt-prediction-lipschitz} gives a clipped prediction envelope with scalar part
    \begin{equation*}
        C_{\star}^{\mathrm{KT}}(\varepsilon)
        \frac{2^t\log(t+2)}{\sqrt{t+1}}.
    \end{equation*}
    For the linear schedule,
    \begin{equation*}
        Q(s,t)=\sum_{r=s}^{t-1}\lceil cr\rceil
        \geq
        c(t-1)
        \qquad(s<t).
    \end{equation*}
    Since \(0<\eta<1\), larger exponents make \(\eta^{Q(s,t)}\) smaller.
    The condition \(c\geq -2\log2/\log\eta\) ensures that the exponential decay from \(\eta^{Q(s,t)}\) cancels the \(2^t\) factor in \(L^{\mathrm{KT}}(t)\) uniformly enough that
    \begin{equation*}
        C_{\star}^{\mathrm{KT}}(\varepsilon)
        \frac{2^t\log(t+2)}{\sqrt{t+1}}
        \sum_{s<t}\eta^{Q(s,t)}
        \leq
        \frac{12C_{\star}^{\mathrm{KT}}(\varepsilon)}{\sqrt t}.
    \end{equation*}
    Applying \zcref{lem:sqrt-disagreement} gives the \(48C_{\star}^{\mathrm{KT}}(\varepsilon)\sqrt N\sqrt T\) bound.
\end{IEEEproof}

The shifted-exponential potential~\eqref{eq:shifted-exp-potential} uses the same route.
Its verification as an excellent coin-betting potential is the exponential-potential calculation of~\cite{orabona_coin_2016} with the time index shifted by one.

\begin{lemma}[Shifted-exponential prediction Lipschitz envelope]\label{lem:shifted-exp-prediction-lipschitz}
    Let \(B_t^{\mathrm{SE}}(G)=\operatorname{Bet}_{t+1}^{\mathrm{SE}}(\clip_t(G))\).
    For every \(t\geq1\),
    \begin{equation*}
        \norm{B_t^{\mathrm{SE}}(a)-B_t^{\mathrm{SE}}(b)}
        \leq
        (2\varepsilon+1)
        \frac{\exp(t/2)}{\sqrt{t+1}}\norm{a-b}.
    \end{equation*}
\end{lemma}
\begin{IEEEproof}
    For the scalar profile \(h_{t+1}^{\mathrm{SE}}(r)=\tanh(r/(t+2))\varepsilon\exp(r^2/(2(t+1)))/\sqrt{t+1}\), the bounds \(\tanh x\le x\), \(\operatorname{sech}^2(x)\le1\), and \(\exp(r^2/(2(t+1)))\le e^{t/2}\) imply
    \begin{equation*}
        |(h_{t+1}^{\mathrm{SE}})'(r)|,\ |h_{t+1}^{\mathrm{SE}}(r)|/r
        \le (2\varepsilon+1)\frac{\exp(t/2)}{\sqrt{t+1}}.
    \end{equation*}
    These estimates control the radial and tangential derivatives, so \zcref{lem:radial-lipschitz-envelope} gives the claimed Lipschitz constant after clipping.
\end{IEEEproof}

\begin{lemma}[Linear shifted-exponential gossip schedule]\label{lem:shifted-exp-linear-disagreement}
    Fix a rate \(0<\eta<1\) and endowment \(\varepsilon>0\).
    Define
    \begin{equation*}
        C_{\star}^{\mathrm{SE}}(\varepsilon)=2\varepsilon+1.
    \end{equation*}
    If \(q(t)=\lceil ct\rceil\) and \(c\geq -3/(2\log\eta)\), then the shifted-exponential disagreement sum is \(O(\sqrt T)\):
    \begin{equation*}
        \begin{aligned}
             & 2\sqrt N\sum_{t=1}^T
            \left(
            C_{\star}^{\mathrm{SE}}(\varepsilon)
            \frac{\exp(t/2)}{\sqrt{t+1}}
            \right)
            \sum_{s<t}\eta^{Q(s,t)}
            \\
             & \qquad\leq
            4e^{3/2}C_{\star}^{\mathrm{SE}}(\varepsilon)\sqrt N\sqrt T .
        \end{aligned}
    \end{equation*}
\end{lemma}
\begin{IEEEproof}
    From \(Q(s,t)\ge c(t-1)\) and \(c\ge-3/(2\log\eta)\), we have \(\eta^c\le e^{-3/2}\), hence
    \begin{equation*}
        \sum_{s<t}\eta^{Q(s,t)}
        \leq
        t\exp\!\left(-\frac32(t-1)\right).
    \end{equation*}
    Multiplying by the shifted-exponential Lipschitz envelope gives
    \begin{equation*}
        C_{\star}^{\mathrm{SE}}(\varepsilon)
        \frac{\exp(t/2)}{\sqrt{t+1}}\sum_{s<t}\eta^{Q(s,t)}
        \le \frac{e^{3/2}C_{\star}^{\mathrm{SE}}(\varepsilon)}{\sqrt t}.
    \end{equation*}
    Apply \zcref{lem:sqrt-disagreement}.
\end{IEEEproof}

\paragraph{Global-loss and prediction-disagreement details}

\begin{lemma}[Global-loss comparison through a centralized prediction]\label{lem:global-loss-centralized-comparison}
    Let \(\bar\ell_t\) be the network loss from \zcref{eq:network-loss}, and define
    \begin{equation*}
        \bar g_t=\frac1N\sum_{m=1}^N g_{m,t}.
    \end{equation*}
    Assume each \(\ell_{m,t}\) is convex and \(1\)-Lipschitz on \(\Hil\), and \(g_{m,t}\in\partial \ell_{m,t}(x_{m,t})\).
    Then for every reference point \(y_t\in\Hil\) and every comparator \(u\in\Hil\),
    \begin{equation*}
        \frac1N\sum_{n=1}^N \bar\ell_t(x_{n,t})-\bar\ell_t(u)
        \leq
        \langle\bar g_t,y_t-u\rangle
        +
        \frac3N\sum_{n=1}^N \norm{x_{n,t}-y_t}.
    \end{equation*}
\end{lemma}

\begin{IEEEproof}
    Because \(\bar\ell_t\) is \(1\)-Lipschitz,
    \(N^{-1}\sum_n\bar\ell_t(x_{n,t})-\bar\ell_t(u)\le
    \bar\ell_t(y_t)-\bar\ell_t(u)+N^{-1}\sum_n\norm{x_{n,t}-y_t}\).
    For each \(m\), insert and subtract \(\ell_{m,t}(x_{m,t})\).
    Convexity at \(x_{m,t}\) gives
    \(\ell_{m,t}(x_{m,t})-\ell_{m,t}(u)\le\langle g_{m,t},x_{m,t}-u\rangle\), while \(1\)-Lipschitzness gives \(\ell_{m,t}(y_t)-\ell_{m,t}(x_{m,t})\le\norm{y_t-x_{m,t}}\).
    Decomposing
    \(\langle g_{m,t},x_{m,t}-u\rangle=\langle g_{m,t},y_t-u\rangle+\langle g_{m,t},x_{m,t}-y_t\rangle\) and using \(\norm{g_{m,t}}\le1\), we obtain
    \begin{equation*}
        \ell_{m,t}(y_t)-\ell_{m,t}(u)
        \le \langle g_{m,t},y_t-u\rangle+2\norm{x_{m,t}-y_t}.
    \end{equation*}
    Averaging over \(m\) and substituting proves the claim.
\end{IEEEproof}

\begin{theorem}[Expected compressed prediction disagreement]\label{thm:compressed-kt-prediction-disagreement}\label{thm:compressed-shifted-exp-prediction-disagreement}
    Assume the expected compressed-gossip hypotheses of \zcref{thm:expected-tracking-contraction}, and put \(\mu=1-\rho\), \(\chi=\mu/4\), and \(\etastoch=1-\gamma\mu/2\).
    For KT, let \(C_{\star}^{\mathrm{KT}}(\varepsilon)=(\varepsilon+1)(1+2/\log2)+1\), \(C_{\mathrm{net}}^{\mathrm{KT}}=(2+\chi)24C_{\star}^{\mathrm{KT}}(\varepsilon)\sqrt N\), and \(q(t)=\lceil ct\rceil\) with \(c\ge-2\log2/\log\etastoch\).
    For shifted-exponential, let \(C_{\star}^{\mathrm{SE}}(\varepsilon)=2\varepsilon+1\), \(C_{\mathrm{net}}^{\mathrm{SE}}=(2+\chi)2e^{3/2}C_{\star}^{\mathrm{SE}}(\varepsilon)\sqrt N\), and \(q(t)=\lceil ct\rceil\) with \(c\ge-3/(2\log\etastoch)\).
    Then, for \(P\in\{\mathrm{KT},\mathrm{SE}\}\), with \(B_t^P(G)=\operatorname{Bet}_{t+1}^{P}(\clip_t(G))\),
    \begin{align*}
        \Exp\sum_{t=0}^{T-1}\frac1N\sum_{n=1}^N
         & \norm{B_t^P(G_{n,t})-\operatorname{Bet}_{t+1}^{P}(\bar G_t)}
        \leq C_{\mathrm{net}}^P\sqrt T .
    \end{align*}
\end{theorem}
\begin{IEEEproof}
    Let \(A^P(t)\) be the corresponding Lipschitz envelope from \zcref{lem:kt-prediction-lipschitz,lem:shifted-exp-prediction-lipschitz}.
    By \zcref{lem:average-tracks}, \(\bar G_t=\sum_{s<t}(-\bar g_s)\) and \(\norm{\bar G_t}\le t\), so the centralized prediction is unclipped.
    The \(t=0\) contribution is zero.
    For \(t\ge1\), \(\operatorname{Bet}_{t+1}^P(\bar G_t)=B_t^P(\bar G_t)\).
    Thus Lipschitzness of \(B_t^P\), followed by the state-disagreement unrolling, gives
    \begin{align*}
         & \Exp\frac1N\sum_n
        \norm{B_t^P(G_{n,t})-\operatorname{Bet}_{t+1}^P(\bar G_t)}
        \\
         & \qquad\leq
        A^P(t)\,
        \Exp\frac1N\sum_n\norm{G_{n,t}-\bar G_t}
        \\
         & \qquad\leq
        (2+\chi)\sqrt N\,A^P(t)
        \sum_{s<t}\etastoch^{Q(s,t)} .
    \end{align*}
    Summing over \(t\) and applying the KT or shifted-exponential schedule estimate gives the two stated constants \(C_{\mathrm{net}}^{\mathrm{KT}}\) and \(C_{\mathrm{net}}^{\mathrm{SE}}\).
\end{IEEEproof}

\begingroup \raggedright{}
\bibliographystyle{IEEEtran}
\bibliography{refs}
\endgroup
\end{document}